\documentclass{llncs}

\usepackage{amssymb}
\usepackage{amsmath}
\usepackage{algorithm}
\usepackage{algpseudocode}
\usepackage[all]{xy}

\usepackage{graphicx}
\usepackage{xspace}

\ifx\pdftexversion\undefined
  \usepackage[dvips,usenames]{color}
\else
  \usepackage[pdftex,usenames,dvipsnames]{color}
\fi
\usepackage{listings}

\usepackage{pgfplots}

\definecolor{lightgray}{rgb}{0.97, 0.97, 0.97}

\lstdefinelanguage{minizinc}
{
morekeywords={
  ann, annotation, any, array, assert, bool, constraint, dominance_nogood,
  else, endif, enum,
exists, float, forall, function,
if, in, include, int, list, of, op, output, par, predicate, record, set,
solve, string, test, then, tuple, type, var, where,
abort, abs, acosh, array_intersect, array_union,
array1d, array2d, array3d, array4d, array5d, array6d, asin, assert, atan, bool2int, card,
ceil, combinator, concat, cos, cosh, dom, dom_array, dom_size, dominance,
fix, exp, floor, index_set, index_set_1of2,
index_set_2of2, index_set_1of3, index_set_2of3, index_set_3of3, int2float, is_fixed,
join, lb, lb_array, length, let, ln, log, log2, log10, min, max, not pow, product, round, set2array,
show, show_int, show_float, sin, sinh, sqrt, sum, tan, tanh, trace, ub, and ub_array,
minisearch, search, while, repeat, next, commit, print, post, sol, scope, time_limit, break, fail
},
sensitive=false, 
morecomment=[l][\em\color{ForestGreen}]{\%},
morestring=[b]",
}

\def\mzninline{\verb}

\newcommand{\sols}{\ensuremath{{\cal S}}}
\newcommand{\vpar}{}

\newcommand{\pjs}[1]{\marginpar{\sc pjs}\textcolor{blue}{#1}}
\newcommand{\tias}[1]{\marginpar{\sc tias}\textcolor{magenta}{#1}}

\newcommand{\ignore}[1]{}

\newcommand{\MiniZinc}{\mbox{MiniZinc}\xspace}
\newcommand{\FlatZinc}{\mbox{FlatZinc}\xspace}

\title{Solution Dominance over Constraint Satisfaction Problems}

\author{Tias Guns\inst{1} \and Peter J. Stuckey\inst{2}\and Guido Tack\inst{2}}
\institute{VUB Brussels, Belgium \\
  \email{tias.guns@vub.be}
  \and
  Data61 CSIRO \& Monash University, Australia \\
  \email{pstuckey@unimelb.edu.au}  \\
  \email{guido.tack@monash.edu}
}

\begin{document}

\maketitle

\begin{abstract}
Constraint Satisfaction Problems (CSPs) typically have many solutions that satisfy all constraints. Often though, some solutions are preferred over others, that is, some solutions \textit{dominate} other solutions.
We present \textit{solution dominance} as a formal framework to reason about such settings. We define Constraint Dominance Problems (CDPs) as CSPs with a dominance relation, that is, a preorder over the solutions of the CSP.
This framework captures
many well-known variants of constraint satisfaction, including optimization, multi-objective optimization, Max-CSP, minimal models, minimum correction subsets as well as optimization over CP-nets and arbitrary dominance relations.

We extend \MiniZinc, a declarative language for modeling CSPs, to CDPs by introducing dominance nogoods; these can be derived from dominance relations in a principled way.
A generic method for solving arbitrary CDPs incrementally calls a CSP solver
and is compatible with any existing solver that supports MiniZinc.
This encourages experimenting with different solution dominance relations for a problem, as well as comparing different solvers without having to modify their implementations.

\end{abstract}

\section{Introduction}
Constraint satisfaction has proven to be an indispensable paradigm for solving complex problems in A.I. and industry. Indeed, many such problems can be expressed as a conjunction of constraints over variables, including logical constrains, mathematical relations and sophisticated \textit{global} constraints such as automata.

However, for many applications, the true problem to be solved is not a satisfaction problem, though satisfaction is a critical component. For example, in many cases one is interested in finding a solution that minimizes an objective function, rather than any satisfying solution.
Many other settings exist in which some satisfying solutions are more interesting or preferred than others, that is, some solutions \textit{dominate} other solutions.

We introduce \textit{solution dominance} as a way to express such dominance relations over the solutions of a CSP. A dominance relation here defines a preorder over the solutions. In line with Constraint Satisfaction Problems and Constraint Optimization Problems, we call the resulting problems Constraint Dominance Problems (CDPs). The goal is to find all non-dominating solutions to the CDP, that is, the Pareto optimal set. We discuss two variants, depending on whether equivalent solutions are allowed in the solution set or not.

Our work generalizes the work on preferences in SAT~\cite{DiRosa:2010:SSP:1842832.1842846}, where a strict partial order over literals is used. This captures MinOne, MaxSAT and Minimum Correction Subsets. Our work generalizes this and other works expressing preference as strict (irreflexive) partial orders, because: 1) (reflexive) preorders give us the freedom to reason both about solution sets that do or do not allow for equivalent solutions; 2) the goal is not to find all dominated (preferred) solutions, but rather all \textit{non-dominated} solutions; hence 3) the investigated approach also capture multi-objective optimization and more. The set of non-dominated solutions is known as the Pareto frontier or the \textit{efficient} set in multi-objective optimization~\cite{Ehrgo00b} and our formalization is inspired by that, but can reason over arbitrary partial orders. 

Preferences in SAT~\cite{DiRosa:2010:SSP:1842832.1842846} are methodologically different from preferences expressed through Conditional Preference networks (CP-nets, \cite{DBLP:journals/jair/BoutilierBDHP04}), because the latter requires expensive dominance checks. We investigate a novel dominance relation for CP-nets, and show that all these types of preferences can fit in the same framework and methodology.
The framework can also express arbitrary solution dominance relations, including other forms of conditional preferences than CP-nets. This is motivated by recently studied data mining problems involving conditional dominance relations over CSPs~\cite{cp4im_domin}. A further discussion of related work is provided in Section~\ref{sec:relwork}.

Inspired by declarative languages for modeling
CSPs~\cite{minizinc,essence,OPL} we propose an extension to the
\MiniZinc modeling language that enables the formulation of CDPs. The idea
is to specify a \textit{dominance nogood}, a constraint pattern that can be used whenever a solution is found during search to exclude dominated solutions in the remaining search (analogous to the well-known branch-and-bound approach for optimization). Dominance nogoods can be derived from the
solution dominance relation in a principled way. 
We present an intuitive implementation of a generic algorithm for dominance nogoods
in the MiniSearch~\cite{Rendl15MiniSearch} meta-search language. We explore the possibilities of this approach on a number of
problems and with a range of different solvers to show the potential of such a generic approach.

\section{CSPs and solution dominance} 
Conceptually, solution dominance can be used to solve problems of the form: find $\{X \in {\cal S} | \nexists Y \in {\cal S} ...\}$ where $\sols$ is the set of all solutions of a CSP.

A \emph{Constraint Satisfaction Problem (CSP)} is a triple $(V,D,C)$ where $V$ is a
set of variables, $D$ is a mapping from variables to a set of values, 
and $C$ is a set of constraints over (a subset of) $V$.
Constraints can represent arbitrary complex relations over the variables.
A valuation $X$ is a mapping of variables to values: $\forall v \in V: v \mapsto D(v)$. 
A \emph{solution} of CSP $(V,D,C)$ is a valuation $X$ that satisfies each constraint $c \in C$. We denote by $X(v)$ the value of $v$ in solution $X$.

\paragraph{Dominance relation.}
A dominance relation over CSP solutions expresses when one solution dominates or is equivalent to another one. A simple example is that of a constrained optimization problem, where 
the optimization function defines a total (pre)order on the solutions. We follow the optimisation convention that smaller is better.

We define a dominance relation $\preceq$ as a preorder over the set of solutions of a CSP $P$. A preorder is a reflexive ($a \preceq a$) and transitive ($a \preceq b \wedge b \preceq c \rightarrow a \preceq c$) relation. It can be thought of as a partial order over equivalence classes.
In other words, given two solutions either one dominates the other, or they are equivalent, or they are incomparable. The use of (reflexive) preorders instead of (irreflexive) partial orders allows us to discriminate between incomparable and equivalent solutions. This is also a key difference between the problem of finding all dominant/preferred solutions and all non-dominated solutions: neither of two equivalent solutions is strictly dominant, while both are non-dominant; incomparable solutions are also non-dominant.

More formally, we define the equivalence relation $X \sim Y \Leftrightarrow X \preceq Y \wedge X \succeq Y$; and the negations $\nsim$ and $\npreceq$.
Now, let $\sols$ be the set of all solutions of a CSP, and $\preceq$ a dominance relation. 
We identify three possible properties for sets $A \subseteq \sols$:
\begin{description} 
  \item[complete:] 
  every solution in $\sols$ is dominated by or equivalent to a solution in $A$: $\forall X \in \sols, \exists Y \in A: Y \preceq X$.
  
  \item[domination-free:] 
  the solutions in $A$ are not dominated by any other in $A$, except equivalent ones: 
  $\forall X,Y \in A: Y \npreceq X \vee X \sim Y$. Equivalently, no $X \in A$ is strictly dominated: $\forall X \in A, \nexists Y \in A: Y \preceq X \wedge X \nsim Y$.
  
  \item[equivalence-free:] 
  no two solutions in $A$ are equivalent to each other: $\forall X,Y \in A: X \nsim Y$. In preference terms, they are \textit{indifferent} to each other.
\end{description}

The set of complete and domination-free solutions is unique and defined as follows: $\{ X \in \sols \mid \forall Y \in \sols, Y \npreceq X \vee X \sim Y \}$. In the multi-objective optimization literature~\cite{Ehrgo00b}, this set is known as the Pareto frontier or the \textit{efficient} set. In that context it is often written in the form $\{ X \in \sols \mid \nexists Y \in \sols: Y \preceq X \wedge X \nsim Y \}$.

In practice though, one is typically just interested in a complete and domination-free set that is also \textit{equivalence-free}: $A \subseteq \{ X \in \sols \mid \forall Y \in \sols, Y \npreceq X \vee X \sim Y \}$ with $\forall X,Y \in A: X \nsim Y$.
This set is not unique (it contains one arbitrary solution per equivalence class). 

\section{Constraint Dominance Problems}
We now show how a wide range of problems that are not captured by the classical CSP framework can be expressed with a dominance relation over a CSP. Generic solving methods are discussed in the next section.

A Constraint Optimization Problem (COP) is typically defined as a tuple
$(V,D,C,f)$ where $(V,D,C)$ is a CSP as defined before and $f$ is a function
over a valuation of the variables $V$. A solution to a COP is a solution to the corresponding CSP that
minimizes the function $f$. One is typically interested in finding one such
optimal solution, that is, one solution $V^*$ of $(V,D,C)$ for which $\nexists V': f(V') < f(V^*)$. 

In line with COPs we define a \textbf{Constraint Dominance Problem} (CDP) as a tuple
$(V,D,C,\preceq)$ where $(V,D,C)$ is a CSP and $\preceq$
a dominance relation. Two types of queries exist for this problem (with or without equivalent solutions). We call the \emph{\textbf{full} solution to the CDP} the set of complete and domination-free solutions of the CSP. A (non-unique) complete and domination-free set that is also equivalence free is simply called a \emph{solution to the CDP}.

\newcommand{\mypara}[1]{\item[#1]}
\begin{description}

\mypara{Optimization.} Given a COP $(V,D,C,f)$, let $\preceq_f$ be the total order corresponding to $f$: $X \preceq_f Y \Leftrightarrow f(X) \leq f(Y)$.
Then, finding a solution to the COP corresponds to finding a solution to the CDP $(V,D,C,\preceq_f)$.

\mypara{Lexicographic optimization.} Let $(V,D,C)$ be a CSP and $F$ a set of functions $\{f_1, \ldots, f_n\}$. The goal is to find a solution that lexicographically minimizes the functions. Given the preorder $\preceq_{lex_F}: X \preceq_{lex_F} Y \Leftrightarrow f_1(X) < f_1(Y) \vee (f_1(X) = f_1(Y) \wedge f_2(X) < f_2(Y) \vee ( \ldots  \wedge f_n(X) \leq f_n(Y) ))$. A solution to the CDP $(V,D,C,\preceq_{lex_F})$ is a lexicographically optimal solution.

\mypara{Multi-objective optimization.} 
The goal in multi-objective optimization is to find all Pareto optimal (non-dominated) solutions given a set of functions. Let $\preceq_{F}$ be the following preorder: $X \preceq_F Y \Leftrightarrow \forall_i f_i(X) \leq f_i(Y)$.
\begin{lemma}
The full solution to the CDP $(V,D,C,\preceq_F)$ corresponds to the set of all Pareto optimal solutions.
\end{lemma}
\begin{proof}
\begin{align*}
				   & \{ X \in S \mid \nexists Y \in S: Y \preceq_F X \wedge X \nsim_F Y \} \\
  \leftrightarrow~ & \{ X \in S \mid \nexists Y \in S: \forall_i f_i(Y) \leq f_i(X) \wedge \neg (\forall_j f_j(X) = f_j(Y)) \} \\
  \leftrightarrow~ & \{ X \in S \mid \nexists Y \in S: \forall_i f_i(Y) \leq f_i(X) \wedge \exists_j f_j(X) \neq f_j(Y) \} \\
  \leftrightarrow~ & \{ X \in S \mid \nexists Y \in S: \forall_i f_i(Y) \leq f_i(X) \wedge \exists_j f_j(X) < f_j(Y) \} \label{eq:mo} 
\end{align*}
which is the classical definition of multi-objective
optimization~\cite{Ehrgo00b}.
\hfill $\Box$
\end{proof}

\mypara{${\cal X}$-minimal models.} Let ${\cal X} \subseteq V$ be Boolean (0/1) variables of a CSP $(V,C,D)$. Given a solution $X$, let $pos_{\cal X}(X) = \{ v
\in {\cal X} | X(v) = 1 \}$ be the variables of ${\cal X}$ assigned to $1$ in solution $X$. An ${\cal X}$-minimal model is a solution to the CSP such that there is no solution $Y$ with $pos_{\cal X}(Y) \subset pos_{\cal X}(X)$.
The set of all $\cal X$-minimal models is: $\{
X \in S | \nexists Y \in S: pos_{\cal X}(Y) \subset pos_{\cal X}(X) \}$.

We can define the preorder $\preceq_{\cal X}$ as: $X \preceq_{\cal X} Y \Leftrightarrow \forall v
\in {\cal X}: X(v) \leq Y(v)$, that is, for each variable of ${\cal X}$ an assignment to $\mathit{false}$
(domain value $0$) is preferred over $\mathit{true}$ (domain value $1$).

\begin{lemma}
The full solution to the CDP $(V,D,C,\preceq_{\cal X})$ corresponds to the
set of ${\cal X}$-minimal models.
\end{lemma}
\begin{proof}
Using the same rewriting as for multi-objective
optimization, we obtain that the set of non-dominated solutions is:
\begin{align*}
				   & \{ X \in S \mid \nexists Y \in S: Y
  \preceq_{\cal X} X \wedge X \nsim_{\cal X} Y \} \\
  \leftrightarrow~ & \{ X \in S \mid \nexists Y \in S: \forall_{v \in {\cal X}} Y(v) \leq X(v) \wedge \nonumber \exists_{v \in {\cal X}} X(v) \neq Y(v) \} \\ 
  \leftrightarrow~ & \{ X \in S \mid \nexists Y \in S: \forall_{v \in {\cal
      X}} pos(Y \cap \{v\}) \subseteq pos(X \cap \{v\}) \nonumber \\ 
&\quad \quad \quad \quad \quad \quad \quad \wedge \exists_{v \in {\cal X}} pos(X \cap \{v\}) \neq pos(Y \cap \{v\}) \} \\
  \leftrightarrow~ & \{ X \in S \mid \nexists Y \in S: pos_{\cal X}(Y) \subseteq pos_{\cal X}(X) \wedge pos_{\cal X}(X) \neq pos_{\cal X}(Y) \}
\end{align*}
which equals $\{ X \in S | \nexists Y \in S: pos_{\cal X}(Y) \subset
pos_{\cal X}(X) \}$. \hfill $\Box$
\end{proof}
The concept of minimal models can be extended to non-Boolean CSPs by defining a total order $\leq_v$ over the possible values of each variable $v \in V$, in line with $X(v) \leq Y(v)$.

\mypara{Weighted (partial) MaxCSP.} Given a CSP $(V,D,C)$ and a function $g : C \rightarrow \mathbb{R}$ that represents the weight of a constraint. Let $X$ be a valuation which need not be a solution to the CSP. The total weight of $X$ is the sum of the weights of the constraints that are satisfied by $X$: $w(X) = \sum_{c \in C, c(X) = true} g(c)$. The goal is to find an assignment to $V$ that maximizes this weight.
As our dominance relation is over solutions of a CSP, we define a new CSP $(V', D', C')$ as follows: a set of $|C|$ new Boolean
variables $B$ is added to $V$: $V' = V \cup B$, and each (soft) constraint is
replaced by a reified version: $C' = \{ B_c \rightarrow c | \forall c
\in C\}$. Let preorder $\preceq_g$ be $X \preceq_g Y \Leftrightarrow w'(g,X) \geq w'(g,Y)$, with $w'(g,X) = \sum_c g(c)*X(B_c)$, then, the weighted MaxCSP problem
$(V,D,C)$ given $g$ is equivalent to the CDP $(V',D',C',\preceq_{g})$. Because $\preceq_{g}$ is a total order, this is equivalent to a COP over $w'$. 

\mypara{Valued CSPs.} A \textit{valued} CSP with
annotated constraints~\cite{DBLP:conf/ijcai/SchiexFV95} has 
a \textit{valuation structure} $(E, \oplus, \leq_v,
\bot, \top)$. Each constraint is mapped to a problem specific value in $E$. The values are aggregated using operator $\oplus$, and $\leq_v$ defines a total over
$E$. Because of the total order, a similar encoding
to that for weighted MaxCSP can be obtained.

\ignore{
\mypara{semiring CSPs} ... uses a c-semiring as the set of levels instead of a
valuation structure ... where operator $\trianglelefteq$ defines partial order.
}

\mypara{Maximally satisfiable subsets.} 
A maximally satisfiable subset $M \subseteq C$ of $(V,D,C)$ is such that $(V,D,M)$ is satisfiable and adding any other constraint leads to unsatisfiability: $\forall c \in C \setminus M: (V,D, M \cup {c})$ is unsatisfiable. Dually, we call $C \setminus M$ a minimal correction subset.

Applying the same transformation of $(V,D,C)$ to $(V',D',C')$ as for MaxCSP problems, the goal is to find solutions $X$ with set $pos_B(X) = \{c \in C | X(B_c) = 1\}$ of active constraints such that no additional constraints in $C$ can be added to it: $\nexists Y \in S: pos_B(Y) \supset pos_B(X)$. This corresponds to finding all $B$-\textit{maximal} models. In line with minimal models, we can define the preorder $X \preceq_{MSS} Y \Leftrightarrow \forall c \in C: X(B_c) \geq Y(B_c)$ and corresponding CDP $(V',D',C',\preceq_{MSS})$.

\begin{figure}[t]
\centering
  \footnotesize
  \newcommand{\xyo}[1]{*+++[o][F-]{#1}}
  \newcommand{\xyd}[1]{*+[F=]{#1}} 
  \newcommand{\xydd}[1]{*+[F.]{#1}}
$$
\begin{array}{l}
\xymatrix@=45pt{
  \xyo{V_1} \ar[r] & \xyo{V_2} \ar[r]  & \xyo{V_3}
  }
\\[10px]
\begin{array}{lll}
\begin{array}{|c|}\hline V_1 \\ \hline 1 < 0 \\ \hline \end{array} &
\quad \quad \quad
\begin{array}{|c|c|}
  \hline V_1 & V_2 \\
  \hline 0 & 1 < 0 < 2 \\
         1 & 2 < 1 < 0 \\ \hline
\end{array} &
\quad \quad
\begin{array}{|c|c|}
  \hline V_2 & V_3 \\
  \hline 0 & 0 < 1 < 2 \\
         1 & 0 < 2 < 1 \\
         2 & 1 < 2 < 0 \\ \hline
\end{array}
\end{array}
\end{array}
$$
\vspace{-1em}
\caption{\small CP-net example over 3 variables.\label{fig:cpnet}}
\vspace{-1em}
\end{figure}
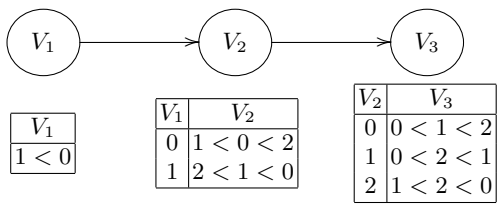

\mypara{CP-nets.}
A CP-net is an acyclic directed graph over a set of variables
$V$~\cite{DBLP:journals/jair/BoutilierBDHP04}. Each node in a CP-net
corresponds to a variable $V_i$ and has a conditional preference table
$\mathit{CPT}(V_i)$. A CPT associates with each possible partial valuation
$a$ of the \textit{parent} variables in the graph, a strict total order
$<_i^{a}$ over the values of $V_i$ (consistent with the rest of this paper,
$x < y$ means $x$ is preferred over $y$). Figure~\ref{fig:cpnet} shows an example.

A CP-net induces a set of \textit{preference rankings} that are consistent with all CPTs, where a preference ranking is a total ordering over all complete valuations of $V$. Traditionally, dominance in a CP-net is defined in terms of its preference rankings: $o$ dominates $o'$ if in all of the preference rankings of the CP-net $o$ is ordered before $o'$. Even for binary-valued CP-nets, the complexity of a dominance check for an arbitrary network is NP-hard \cite{DBLP:journals/jair/BoutilierBDHP04}.

An easier to compute query is the \textit{ordering query}~\cite{DBLP:journals/jair/BoutilierBDHP04}: given two valuations $o$ and $o'$, is there a CPT where for all ancestor variables $o$ and $o'$ have the same value and $o$ is preferred to $o'$ according to the CPT: $\exists v \in V: (\forall w \in Ancestor_{\cal N}(v), o(w) = o'(w) \wedge o <_{\mathit{CPT}_{\cal N}(v)} o')$? If so, there must exist a preference ordering where $o < o'$ and hence this implies that $o'$ does not dominate $o$. It only \textit{implies} non-dominance though (sufficient but not necessary), meaning that there may not exists such a CPT and yet $o'$ does not dominate $o$. This is because only CPTs are checked for which all \textbf{ancestors} have the same value in $o$ and $o'$; hence there may be a CPT with $Ancestor_{\cal N}(v) \neq Parent_{\cal N}(v)$ that induces preference rankings in which $o < o'$ but the CPT is not checked because of the ancestor condition.

We hence propose the following \textit{weaker} form of CP-net dominance:
\begin{definition}\label{def:cpnet_dom}
Given CP-net $\cal N$, a valuation $o$ \textit{locally} dominates $o'$ iff for all CPTs: when its \textbf{parents} have equal value in $o$ and $o'$ but the CPT variable does not, then $o$ must be preferred to $o'$ by the CPT;
\begin{eqnarray}
o \prec_{\cal N} o' \Leftrightarrow \forall v \in V: (&&(\forall w \in Parents_{\cal N}(v), o(w) = o'(w) \wedge o(v) \neq o'(v)) \nonumber \\
&&\rightarrow o <_{\mathit{CPT}_{\cal N}(v)} o'~)
\end{eqnarray}
\end{definition}

Note that there is always at least one $v$ with corresponding CPT active, because a CP-net is acyclic and hence there is at least one node with an empty parent set.

\begin{lemma}
\textit{Local} dominance in a CP-net is a necessary but not a sufficient condition for traditional (\textit{preference ranking}-based) dominance.
\end{lemma}
\begin{proof}
All preference rankings have to agree with all CPTs. Hence, if $o$ dominates $o'$ then for all applicable (parent variables equal, current variable not) CPTs: $o <_{\mathit{CPT}_{\cal N}(v)} o'$, so local dominance must necessarily be true. However, if $o \nprec_{\cal N} o'$ then by negation there must exists a node $v$ with applicable CPT for which $o \nless_{\mathit{CPT}_{\cal N}(v)} o'$ and hence because of $o(v) \neq o(v')$: $o' <_{\mathit{CPT}_{\cal N}(v)} o$. However, it may still be the case that $o$ dominates $o'$, because of an interaction between the grandparent variables and CPTs such that no preference ranking actually uses this entry of $\mathit{CPT}_{\cal N}(v)$. This cannot be observed from the \textit{local} CPTs directly, but must be verified through a (NP-hard) dominance check over the preference rankings.
\end{proof}

\begin{definition}
Because each CPT in a CP-net encodes strict relations, two valuations can only be equivalent if they are identical: $Y \sim_{\cal N} X \equiv Y = X$.
We hence define $Y \preceq_{\cal N} X \equiv Y \prec_{\cal N} X \vee Y = X$.
\end{definition}

Local dominance can be checked in time linear in the number of CPTs in the CP-net. We can hence realistically use it to enumerate all non-locally-dominated solutions of a CP-net by solving CDP $(V,D,C,\preceq_{\cal N})$. The resulting solution will be an over-approximation of the actual set of traditional non-dominated solutions. Should the application demand it, one could still filter the resulting set by post-processing with the NP-hard traditional dominance check.

\ignore{
\begin{figure}
\centering
\includegraphics[width=0.5\linewidth]{cpnet}
\caption{CP-net example over 3 variables.\label{fig:cpnet}}
\end{figure}
}

\mypara{Domain-specific dominance relations.} We showcase the need for
domain-specific dominance relations in data mining. Increasingly, constraint
programming is used for data mining problems such as searching for
\textit{patterns} that appear \textit{frequently} in a
database~\cite{Guns2015}. A pattern can be a set of items, a sequence or
another structure such as a graph~\cite{cp4im_struct}. A pattern is
\textit{frequent} if it is a subpattern of sufficiently many objects in the
database. The problem is typically encoded as a set of constraints that
define the pattern type, and that define when a pattern is frequent. One
solution to this CSP is then one frequent pattern.

However, there are a number of pattern mining settings that do not fit the
classical CSP framework, most notably \textit{closed} and \textit{maximal}
patterns, \textit{relevant} patterns and \textit{skyline}
patterns.
Using the concept of dominance though, these problems can be modeled declaratively and combined with arbitrary constraints~\cite{cp4im_domin}.

For example, a frequent pattern is \textit{maximal} if there is no other
frequent pattern that is a \textit{superpattern} of this pattern. Let
$\sqsubseteq$ represent the subpattern relation, e.g. subset, subsequence or
subgraph relation depending on the pattern type. The general dominance
relation for maximal patterns is: $X \preceq_{maximal} Y \equiv X
\sqsupseteq Y$. Let $S$ be the set of all frequent patterns, then the full
solutions is $\{ X \in S \mid \nexists Y \in S: Y \sqsupseteq X \wedge X
\nsim Y \} \equiv \{ X \in S \mid \nexists Y \in S: Y \sqsupset X \}$, the
set of all \textit{maximally} large patterns such that no other frequent
pattern is a superpattern of it. For patterns represented by a set $I$, the
subpattern relation is the subset relation over $I$, hence; $\{ X \in S \mid
\nexists Y \in S: pos_{I}(Y) \supset pos_{I}(X) \}$.
Note that this is
equivalent to an ${\cal X}$-maximal model with ${\cal X} = I$.

\textit{Closed frequent} patterns have the weaker condition that there should not be a superpattern with the same frequency: 
dominance relation $X \preceq_{closed} Y \equiv X \sqsupseteq Y \wedge freq(X) = freq(Y)$. 
The resulting full solution set is the set of all frequent patterns for which no superpattern exists that has the same frequency: $\{ X \in S \mid \nexists Y \in S: Y \sqsupset X \wedge freq(Y) = freq(X) \}$, this is true for all pattern types.
An \textit{algebra} inspired by databases is used in \cite{cp4im_domin} to
show how a number of other pattern mining problems can be expressed as a
combination of a constraint algebra and a dominance algebra. The framework
we propose here focuses on one CSP and one dominance relation, which is
sufficient for most settings.

\end{description}

\section{Search and dominance nogoods} \label{sec:domnogood}
The main task that we consider is to find a solution to the constraint dominance problem $(V,D,C,\preceq)$, that is, a complete, domination-free and optionally equivalence-free set of solutions. Our methodology is to solve the problem through a chain of constraint satisfaction problems, similar in spirit to \cite{DiRosa:2010:SSP:1842832.1842846}.

Let ${\cal O}(V,D,C)$ be an oracle that returns a satisfying solution to the CSP $(V,D,C)$ or fails if no such solution exists. 

\vpar
\paragraph{Complete.} the following algorithm returns a complete solution to a CSP $(V,D,C)$ with dominance relation $\preceq$, using oracle ${\cal O}$:

\begin{algorithm}
\caption{search($V,D,C,\preceq, {\cal O}$): 
  \label{alg:dom}}
\begin{algorithmic}[1]
\State{$A:= \emptyset$}
\While{$S:={\cal O}(V,D,C)$ \label{alg:dom:loop}}
\State{$A:= A \cup \{S\}$}
\State{$C:= C \cup \{ S \npreceq V \vee S \sim V\}$ \label{alg:dom:cons}}
\EndWhile
\State{{\bf return} $A$}
\end{algorithmic}%
\end{algorithm}%
Note how on line~\ref{alg:dom:cons}, a constraint is added over variables $V$ such that any (future) assignment to these variables may not be strictly dominated by the found solution $S$. As the constraint set $C$ is monotonically increasing, the oracle can be \textit{incremental} in that it can continue its search from where its last solution was found; the result is a branch-and-bound style algorithm where the bound is represented by $S \npreceq V \vee S \sim V$.

Let $\langle S_1, S_2, \ldots, S_n \rangle$ be the sequence of solutions as found by the oracle in Algorithm~\ref{alg:dom}. 
\begin{theorem}
  The set $A$ returned by Algorithm~\ref{alg:dom} is a \textit{complete} set: $\forall X \in (V,D,C),$ $\exists Y \in A: Y \preceq X$.
\end{theorem}
\begin{proof}
The set is complete: let $P_x$ be the CSP $(V,D,C)$ that yields solution
$S_x$, let ${\cal S}(P_x)$ be the set of \textit{all} solutions of
$P_x$. $P_1$ is the original CSP and hence ${\cal S}(P_1)$ contains all
solutions of the CSP and this set is complete. The subsequent $P_2$ only
forbids solutions that are dominated and not equivalent to $S_1$, hence
$\{S_1\} \cup {\cal S}(P_{2})$ is also complete as any solution is either
dominated by $S_1$ or in ${\cal S}(P_{2})$. By induction, any set
$\{S_1,S_2,\ldots,S_x\} \cup {\cal S}(P_{x+1})$ is complete. Hence, $S =
\{S_1, \ldots, S_n\} = \{S_1, \ldots, S_n\} \cup {\cal S}(P_{n+1})$ is
complete as the stopping criterion on line~\ref{alg:dom:loop} dictates that
${\cal S}(P_{n+1}) = \emptyset$.
\hfill $\Box$
\end{proof}

\vpar
\paragraph{Complete and equivalence-free.} Algorithm~\ref{alg:dom} can easily be modified to return a complete and equivalence-free set. In this case, line~\ref{alg:dom:cons} has to be changed to $C:= C \cup \{ S \npreceq V\}$. This set is equivalence-free as now any solution to ${\cal S}(P_{x+1})$ cannot be equivalent to $S_x$. The set is still complete as only equivalent solutions are additionally removed by the modification, and hence for each solution $X$ in $(V,D,C)$ it will still be the case that $\exists Y \in A: Y \preceq X$.

\vpar
\paragraph{Domination-free.}
Any solution $S_y$ found 
after $S_x$ cannot be strictly dominated by $S_x$, as this is explicitly forbidden by the added constraint. 
That means that if the oracle enumerates the solutions from most to least preferred according to the preorder (e.g. first assign variables to 1, then to 0 for MaxCSP), then the complete and equivalence-free set is also domination-free. This is the approach used by \cite{DiRosa:2010:SSP:1842832.1842846}.

However, if we assume no order on the solutions found by the oracle, it is possible to find, with $y > x$, an $S_y$ that strictly dominates $S_x$: $S_y \preceq S_x \wedge S_y \nsim S_x$. We can remove these by doing a \textit{backwards pass} over the solutions in which we check for each $S_y, S_x, y > x$ that $S_y \npreceq S_x \vee S_y \sim S_x$ and if not we drop $S_x$.

Note that we reuse oracle ${\cal O}$ for this though it merely has to \textit{check} whether a fixed assignment to the variables satisfies the constraints.

The same procedure can be applied to a complete and equivalence-free set to
make it domination-free; this set is already equivalence-free so the process
will only remove the strictly dominated solutions.

\paragraph{From dominance relation to dominance nogood.}
We refer to the constraint added on line~\ref{alg:dom:cons} of Algorithm 1 as the dominance nogood. It can be derived from the preorder $\preceq$ in a principled way. We will demonstrate this for a number of earlier examples.

Since one is often interested in a complete, domination-free and equivalence-free set, we will demonstrate it for the equivalence-free dominance nogood $S \npreceq V$. We denote the relation representing the dominance nogood as $D(S,V)$. It can often be obtained by negating the logical definition of the preorder $\preceq$. The dominance nogood with equivalence $S \npreceq V \vee S \sim V$ can be obtained following similar methods as well, perhaps more easily in the equivalent form $S \npreceq V \vee V \preceq S$. Recall that $V$ is the set of variables of the CSP and $S$ is a previously found solution.

\ignore{
Let \texttt{R(X,Y)} be the constraint enforcing a relation $X \preceq Y$; for example $X \preceq Y \equiv f(X) \leq f(Y)$ for some function $f()$, then \texttt{R(X,Y) $\equiv$ f(X) $\leq$ f(Y)}.
Let us consider a CSP $P$ as a set of constraints over an implicitly defined set of variables $var(P)$; a constraint $R(X,Y)$ where $X$ and $Y$ are sets of variables such that $vars(X) = vars(Y) = vars(P)$\footnote{does it make sense that vars(X)=vars(Y)?}; and an oracle ${\cal O}(P)$ that, given any CSP $P$, returns one solution to the CSP, unless no such solution exists. This oracle can be used to find a set of complete and minimal solutions over $P$ and $R$ as follows:
}

\begin{description}

\mypara{Optimization.} For dominance relation $X \preceq_f Y \Leftrightarrow f(X) \leq f(Y)$ the dominance nogood is $D(S,V) \Leftrightarrow \neg (f(S) \leq f(V)) \Leftrightarrow f(S) > f(V) \Leftrightarrow f(V) < f(S)$. This guarantees that every new assignment to $V$ will have a smaller score than the previously found solution $S$.

\mypara{Lexicographic optimization.} The dominance relation is $X \preceq_{lex_F} Y \Leftrightarrow f_1(X) < f_1(Y) \vee (f_1(X) = f_1(Y) \wedge (f_2(X) < f_2(Y) \vee ( \ldots \wedge f_n(X) \leq f_n(Y) )))$. The negation is the dominance nogood: $D(S,V) \Leftrightarrow f_1(S) \geq f_1(V) \wedge (f_1(S) \neq f_1(V) \vee (f_2(S) \geq f_2(V) \wedge ( \ldots \vee f_n(S) > f_n(V) )))$. Using the observation that $(A \geq B) \wedge ((A \neq B) \vee Z) \Leftrightarrow (B < A) \vee ((B = A) \wedge Z)$ we obtain $D(S,V) \Leftrightarrow f_1(V) < f_1(S) \vee (f_1(V) = f_1(V) \wedge (f_2(V) < f_2(S) \vee ( \ldots \wedge f_n(V) < f_n(S) )))$. One could also use a \mzninline|lex_less| global constraint.

\mypara{Multi-objective optimization}. The dominance relation $X \preceq_F Y \Leftrightarrow \forall_i f_i(X) \leq f_i(Y)$ has the dominance nogood $D(S,V) \Leftrightarrow \exists_i f_i(S) > f_i(V)$. Recall that as in the previous examples, this is the dominance nogood to obtain the \textit{equivalence-free} set of complete and domination-free solutions.

\mypara{${\cal X}$-minimal models} We have
$X \preceq_{\cal X} Y \Leftrightarrow \forall v \in {\cal X}: X(v) \leq Y(v) \Leftrightarrow pos_{\cal X}(X) \subseteq pos_{\cal X}(Y)$
and dominance nogood $D(S,V) \Leftrightarrow \exists v \in {\cal X}: S(v) > V(v)$.

\ignore{
\mypara{semiring CSPs} ... uses a c-semiring as the set of levels instead of a
valuation structure ... where operator $\trianglelefteq$ defines partial order.
}

\mypara{Minimal correction subsets} We noted earlier that minimal correction subsets is dual to finding the maximal satisfiable subsets, which corresponds to the problem of finding all maximal models of $(V',D',C')$ with $X \preceq_{MSS} Y \Leftrightarrow \forall c \in C: X(B_c) \geq Y(B_c)$. The corresponding dominance nogood is $D(S,V) \Leftrightarrow \exists c \in C: S(B_c) < V(B_c)$.

\mypara{CP-nets} For CP-nets we have $D(S,V) \Leftrightarrow S \neq V \wedge S \nprec_{\cal N} V ~ \Leftrightarrow ~ S \neq V \wedge \exists v: (\forall w \in Parents(v), S(w) = V(w) \wedge S(v) \neq V(w) \wedge V <_{\mathit{CPT}(v)} S)$.

Each line in a CPT is mutually-exclusive, so the relation that $V$ is more preferred to $S$ according to $\mathit{CPT}(V_i)$, $V <_{\mathit{CPT}(v)} S$, can be formalized using a disjunction over all CPT entries. The preference in each entry can be modeled using implications. For example, for the first entry of $\mathit{CPT}(X_3)$ in Figure~\ref{fig:cpnet} as follows: $V_2 = 0 \wedge S_2 = 0 \wedge V_3 \neq S_3 \wedge ( (S_3=2 \rightarrow (V_3=1 \vee V_3=0)) \vee (S_3=1 \rightarrow V_3=0) )$. Indeed, the parents have to take a specific (identical) value, and in that case if $S_3 = 2$ then $V$ is preferred only if $V_3$ takes value $1$ or $0$; if $S_3 =1$ then $V_3$ must be $0$ to be preferred.
Because of the mutual exclusivity and because $S$ is a solution of which we know the values, when posting the dominance nogood, at most one entry per CPT will be part of the logical formula.

\end{description}

\paragraph{Complexity.}
Each time the oracle $\cal O$ is called, it has to solve an NP-hard CSP $(V,D,C)$ in general. Furthermore, there can be an exponential number of non-dominated solutions and hence calls to the oracle: there exists a search order such that the algorithm has to enumerate all
solutions to find all non-dominant ones.  Brafman~\emph{et al}~\cite{kr2010} show that the simpler problem of
computing the \emph{next solution} in a CSP or preference problem is in
general also NP-hard, although there are some special cases (like tree CSPs)
where it can be easier.

Note that since the constraint
system is monotonically increasing, any learned nogoods from previous runs
(from a learning solver~\cite{lcg}, or from restarts~\cite{lecoutre})
are valid, hence each subsequent
call to the solve oracle $\cal O$ can take advantage of them.

\section{Integration in a modeling language}
Any constraint solver that supports incrementally adding constraints and retrieving the next solution can implement Algorithm~\ref{alg:dom}. From the modeling perspective, we describe how \MiniZinc~\cite{minizinc}, a modeling language for CSPs and COPs, can be extended to handle CDPs as new modeling primitive.

Instead of extending \MiniZinc to express dominance relations, we instead add \textbf{syntax for dominance nogoods}, for two reasons: 1) users can explicitly specify either an equivalence-free dominance nogood, or a dominance nogood with equivalence;
2) we found it more intuitive to declare an \emph{invariant} for the search (e.g. minimization as $f(V) < f(S)$ where $S$ is a previously found solution), rather then declaring when a previous solution dominates or is equivalent to a new one (e.g. $f(S) \leq f(V)$).

\MiniZinc has keywords to define variables, constraints, predicates and a minimal search specification.
We add the keyword \lstinline|dominance_nogood| for specifying dominance
nogoods.
Furthermore, the built-in MiniSearch~\cite{Rendl15MiniSearch} function \lstinline|sol(X)| can be used to refer to the value of variable $X$ in a previously found solution. Hence, the dominance nogood for a minimization problem: $D(S,V) \Leftrightarrow f(V)< f(S)$ can be expressed as:\\
\indent \lstinline|dominance_nogood f(V) < f(sol(V));|\\
with $V$ some array of variables of the accompanying CSP.

Here is an example for finding minimal correction subsets 
(note the \lstinline|not| in the output statement, where we convert the maximal satisfiable subset to a minimal correction subset).
\begin{lstlisting}
array [int] of var bool: B;
constraint B[1] -> ...;
dominance_nogood exists(i in index_set(B))(B[i] < sol(B[i]));
output ["MCS:"] ++ ["\(i) " | i in index_set(B)
                              where not fix(B[i])]; 
\end{lstlisting}

\ignore{
\textit{
Optional: syntactic extension to the language for multi-objective?
E.g. \lstinline|solve maximize [f(X), g(X), -h(X)]|?? Would be translated to
dominance expression in pre-processing step.
}
\pjs{We should really do this, and indeed at the same time expand
  \texttt{<}, \texttt{<=}
  to apply to arrays, meaning lexicographic order. Thus replacing
  \texttt{lex\_less} etc. Then solve minimize [f(X), g(X), -h(X)] is
  transparently
  the same as other minimizations!}
\tias{For lex yes, but I was thinking of multi-objective here.}
}


\vpar
\subsection{Solving}
Algorithm~\ref{alg:dom} can be specified in a straightforward way using the recently introduced MiniSearch language~\cite{Rendl15MiniSearch}. MiniSearch is an extension of \MiniZinc with support for specifying meta-level search heuristics. The integration with \MiniZinc allows any existing \FlatZinc solver to be used by MiniSearch, and hence also for solving CDPs.

At the language level, a declaration \lstinline|dominance_nogood e| is simply translated into a predicate declaration \lstinline|predicate post_dng() = e| that posts the dominance nogood. Algorithm~\ref{alg:dom} is then specified as follows in MiniSearch:
\begin{lstlisting}
solve search dominance_search;
function ann: dominance_search() =
  repeat( if next() then
            commit() /\ print() /\ post_dng()
          else break endif );
\end{lstlisting}
$ $

The \lstinline|dominance_search| function repeatedly queries a black-box solver for the next solution (\lstinline|next()|), and if a solution is found, it remembers it (\lstinline|commit()|), prints it, and then posts the dominance nogood before continuing to search for the next solution. In MiniSearch, \lstinline|next()| can either call an external solver process through a file based interface, restarting the search from scratch for each call; or use an incremental C++ API that permits adding constraints during the search.

\section{Experiments}
We evaluate the viability of the framework and the opportunities of modeling constraint dominance problems in a declarative solver-independent language. We do not aim to beat the state-of-the-art on any one specific task, as they typically employ specialized bounds or additional inference mechanisms.
The experiments below evaluate different categories of dominance nogoods and show that they can be handled through the generic framework presented in this paper, including the novel setting of optimizing over CP-nets using generic CSP solvers.

The experiments use MiniSearch with the following solvers: Gecode 4.4.0
(\textsf{gecode}),
or-tools v2015-09 (\textsf{ortools}),
Opturion CPX 1.0.2 (\textsf{optcpx})
and
Chuffed b776ac2 (\textsf{chuffed}).
Gecode and or-tools are classical depth-first search CP solvers, while
Opturion CPX and Chuffed are lazy clause generation solvers~\cite{lcg}.
All solvers are called by MiniSearch as external processes using the standard file-based interface. Gecode additionally supports the direct, incremental C++ API (\textsf{gecode-api}).

\subsection{MaxCSP}
We use benchmarks from the 2008 XCSP competition, MaxCSP \textit{with globals} category. Table \ref{tab:maxcsp} shows a comparison of MiniSearch with different solvers on a selection of instances. For each solver, we compare a model where no variable order is given ('free') with the specification of a most-to-least preferred strategy over the $B$ variables only ('ord', e.g. assign to $1$ before $0$). The latter forces the solvers to first consider all constraints, then all but one (arbitrary) constraint, and so on.

\begin{table}[t]
\caption{MaxCSP runtimes in seconds, --- timed out after 30 min.\label{tab:maxcsp}}
\centering
\scalebox{1}{
\begin{tabular}{l||r|r||r|r||r|r||r|r||r|r}
Instance & \multicolumn{2}{c||}{\sf gecode-api} & \multicolumn{2}{c||}{\sf
  gecode} & \multicolumn{2}{c||}{\sf ortools} &
\multicolumn{2}{c||}{\sf chuffed} & \multicolumn{2}{c}{\sf optcpx}\\
 & free & ord & free & ord & free & ord & free & ord & free & ord\\
\hline
\small \tt cabinet-5570 & --- & 0.9 & --- & --- & 36 & 0.2 & 257 & --- & 3.9 & 0.3\\
\small \tt cabinet-5571 & --- & 0.9 & --- & --- & 36 & 0.2 & 257 & --- & 3.9 & 0.4\\
\small \tt latinSq-dg-3\_all & 0.2 & 0.1 & 0.5 & 0.3 & 0.1 & 0.1 & 0.2 & 0.3 & 0.1 & 0.1\\
\small \tt latinSq-dg-4\_all & 0.6 & 0.9 & 0.8 & 6.8 & 0.5 & 1.3 & 0.5 & 13 & 0.6 & 0.3\\
\small \tt quasigrp4-4 & 46 & --- & --- & --- & 4.5 & --- & 3.8 & 18 & 1.4 & 7.7\\
\small \tt quasigrp5-4 & 0.4 & 1651 & 1158 & --- & 1.1 & --- & 1.6 & 5.4 & 1.6 & 1.3\\
\small \tt q13-1110973670 & 479 & 1.1 & 32 & 0.9 & 540 & 0.7 & 635 & 43 & 11 & 7.5\\
\small \tt q13-1111219348 & 569 & 1.1 & 32 & 1.3 & 385 & 0.9 & 641 & 72 & 8.8 & 7.0\\

\end{tabular}
}
\end{table}

We can see in Table~\ref{tab:maxcsp} that providing the search order often
leads to improved runtimes, but not always (\texttt{quasigrp} for
\textsf{gecode}, \texttt{cabinet} for \textsf{chuffed}).
\textsf{gecode} is slower than the incremental API approach
\textsf{gecode-api}
in case an order is given; when doing free search, the restarts of the
file-based approach seem to improve runtime for \texttt{q13} (and others,
not shown).
The remaining solvers seem to handle this task pretty well, especially
\textsf{optcpx}.
For a rough comparison, in the 2008 competition the \texttt{quasigrp} and
\texttt{latinSq}
instances were also solved within seconds, however runtimes of 600+ seconds
were reported for \texttt{cabinet}
and 40+ seconds for \texttt{q13}.

\subsection{Multi-objective}
We consider traveling saleseman problems where two different costs are given between any two cities, e.g. duration and fuel cost. 
We report on the generation of all complete and equivalence-free solutions. The straightforward \textit{backward pass} needed to make the set domination-free is omitted for reasons of simplicity. The instances are from the Oscar repository~~\cite{oscar}. 

\begin{table}[b]
\centering
\caption{Runtime and number of solutions (forward pass) for
  multi-objective optimization; --- indicates time out after 30 min;
  n.a. that the search
  strategy was not supported\label{tab:mo}}
\scalebox{1}{
\begin{tabular}{l||r|r||r|r||r|r||r|r||r|r}
Instance & \multicolumn{2}{c||}{\sf gecode-api} & \multicolumn{2}{c||}{\sf
  gecode} & \multicolumn{2}{c||}{\sf ortools} & \multicolumn{2}{c||}{\sf
  chuffed} & \multicolumn{2}{c}{\sf oscar}\\
 & time & sols & time & sols & time & sols & time & sols & time & sols\\
\hline
ren10 & 0.5 & 108 & 7.5 & 108 & 6.8 & 108 & 38 & 105 & 2.3 & 110 \\
ren15 & 368 & 949 & --- & 545 & -- & 565 & --- & 343 & 61 & 891 \\
ren20 & --- & 998 & --- & 382 & -- & 392 & --- & 381 & --- & --- \\ \hline
ren10-mg & 1.8 & 41 & 2.8 & 41 & 1.3 & 45 & 5 & 38 & n.a. & n.a. \\
ren15-mg & 14 & 135 & 247 & 135 & 541 & 145 & --- & 128 & n.a. & n.a. \\
ren20-mg & --- & 925 & --- & 292 & --- & 294 & --- & 171 & n.a. & n.a.
\end{tabular}
}
\end{table}

In Table \ref{tab:mo} we compare the MiniSearch approach with different
solvers and Oscar~\cite{oscar}, which has an efficient dedicated propagator for
multi-objective optimisation~\cite{oscar-chaus-biobj}. The
first three lines use the first-fail variable ordering used in Oscar, the
last three use a \textit{max-regret} ordering over the distance variables,
as found in MiniZinc's TSP models. The gecode-api results indicate that
file-based restarts lead to much slower solving times. The number of
intermediate solutions also has a big influence on runtime, as using a
better variable order leads to both smaller solution sets and smaller
runtimes.

\vpar
\subsection{CP-net}

The following experiments consider a variant of the Photo problem, where the goal is to find an ordering of friends such that the number of preferences regarding whom to stand next to for a group photo is maximized. We here consider the case that preferences are supplied as a CP-net: each person indicates a number of people (parents in the CP-net) and their preferences considering the locations of these people. Such a CP-net can be partial, i.e., it can contain disconnected components, which our method can handle without any modification.

We randomly generated CP-nets with $n$ people ($5\leq n\leq 20$) and for each person between 0 and $k$ parents (sampled uniformly per person, $1\leq k \leq 10$). The induced order in the CPTs corresponds to preferring smaller average distance to the parent(s). We again report only the forward pass of computing all complete and equivalence-free solutions. Larger $n$ and larger $k$ lead to larger CP-nets and runtimes, though that depends very much on the actual CP-net generated.

\begin{table}[t]
\caption{CP-net photo-like setting, runtimes and number of solutions (forward pass), --- indicates timeout after 30 min.\label{tab:cpnet}}
\centering
\scalebox{1}{
\hspace{-20px}
\begin{tabular}{l||r|r||r|r||r|r||r|r||r|r}
\centering
Instance & \multicolumn{2}{c||}{\sf gecode-api} & \multicolumn{2}{c||}{\sf
  gecode} & \multicolumn{2}{c||}{\sf ortools} & \multicolumn{2}{c||}{\sf
  chuffed} & \multicolumn{2}{c}{\sf optcpx}\\
 & time & sols & time & sols & time & sols & time & sols & time & sols\\
\hline
\small \tt 10v-4p-1 & 109 & 76 & 126 & 78 & 144 & 81 & 110 & 61 & 192 & 112\\
\small \tt 10v-4p-2 & 14 & 14 & 46 & 43 & 17 & 16 & 7 & 7 & 20 & 18\\
\small \tt 10v-4p-3 & 31 & 29 & 51 & 43 & 14 & 12 & 80 & 57 & 78 & 61\\
\small \tt 10v-4p-4 & 397 & 673 & 782 & 285 & 496 & 331 & 547 & 378 & --- & 753\\
\small \tt 10v-4p-5 & 15 & 68 & 101 & 289 & 34 & 95 & 14 & 42 & 24 & 76\\
\end{tabular}
}
\end{table}
 
Table \ref{tab:cpnet} shows results on 5 different networks generated with $n=10$ and $k=4$. No variable/value ordering strategy was imposed. The number of solutions found clearly has an influence on runtime, where the number of solutions is not only specific to the problem at hand, but also to the search order chosen by the solver. We expect a method that uses the expensive (traditional) dominance checks to perform worse.

\section{Related Work and discussion} \label{sec:relwork}

As discussed in the introduction, most related is the work on preferences in SAT~\cite{DBLP:conf/ecai/RosaGM08,DiRosa:2010:SSP:1842832.1842846}, where a preference can be defined over individual literals. They identify \textit{preference formulas} for these tasks, which correspond to dominance nogoods, and use incremental SAT solving. Our framework generalizes this to a wider range of tasks and different solving technology.

As with any generic method, one cannot expect to obtain the most efficient method for each of the covered tasks. Indeed, specialized methods have been developed for MaxSAT~\cite{DBLP:conf/sat/DaviesB13}, minimum correction subsets~\cite{DBLP:conf/ijcai/Marques-SilvaHJPB13} and X-minimal models~\cite{BenEliyahu–Zohary20051} that can be more efficient than a SAT with preferences approach, but only for their specific task.
For multi-objective optimization in CSPs, specialized propagation algorithms exist that filter the search space more effectively~\cite{DBLP:conf/ecai/Gavanelli02,oscar-chaus-biobj}. Similarly for other forms of preference~\cite{junker2004preference}.

Nevertheless, recent applications of data mining using SAT~\cite{DBLP:conf/pkdd/JabbourSS13} and CSP~\cite{cp4im_domin} demonstrate the need for generic methods for handling novel \textit{solution dominance} settings, for example involving \textit{conditional} dominance relations.

Many works in CP-nets focus on consistency and dominance
testing~\cite{Domshlak02cp-nets-,DBLP:journals/jair/BoutilierBDHP04,Boutilier04preference-basedconstrained}. A
branch-and-bound style algorithm for finding all non-dominated solutions
given additional constraints has been
studied~\cite{Boutilier04preference-basedconstrained}; it uses expensive (PSPACE complete~\cite{DBLP:journals/jair/GoldsmithLTW08}) dominance checks in case the ordering query returns false. Furthermore the search (variable order) in their method is driven by the CP-net's structure. In contrast, our method proposes a novel dominance relation that is cheap to evaluate and provides an over-approximation; it can be used with any existing solver and for any variable order.

A recent extension of Answer Set Programming~\cite{DBLP:conf/aaai/BrewkaD0S15} covers some of the tasks in this paper too (no CP-nets or domain-specific relations), but within the stable model semantics. They provide a language extension for expressing preference relations with a preference type (e.g. less, subset, pareto) and preference elements (the variables).
Our language extension is closer to the original dominance relation which can make it easier to specify domain-specific dominance nogoods.

The concept of \textit{dominance} is also used in different contexts in the constraint programming community.
Dominance breaking~\cite{domjournal} for COPs
creates constraints
that, given a mapping $\sigma$, prevent the finding of
solutions $\theta$ such that $\sigma(\theta)$ is a
better solution of the COP.
They can drastically improve solving performance.
\cite{domjournal} rely on a notion of dominance relation that applies to all
valuations and not just solutions, so that sub-trees can be
pruned during the search.
In that sense they are complementary to the solution dominance we consider
in this paper.  Indeed an interesting direction for further work is to
extend dominance breaking to arbitrary solution dominance problems.

\section{Conclusions}

We introduced the concept of \textit{solution dominance}, where the dominance relation is a preorder over the solutions of a CSP.
We call the resulting problems Constraint Dominance Problems, and this captures single/lexicographic/multi-objective optimization, X-minimal models, weighted MaxCSP, minimum correction subsets as well as a novel dominance relation for reasoning over CP-nets, as well as other dominance relations. 
We provide a natural and declarative extension to \MiniZinc for specifying Constraint Dominance Problems, based on MiniSearch.

Preferences and (solution) dominance have a history in CP research, as discussed above. Two directions for future work hence emerge: 1) which other preference~\cite{junker2004preference} or dynamic solving settings~\cite{10.1007/978-3-319-07046-9_6} fit the solution dominance framework; and 2) what specialised solving methods that have been investigated for one category (e.g. MaxCSP or Multi-objective) can also be applied to other settings? From a modeling perspective one can also ask the question whether it can be automatically detected from a dominance nogood specification, that a more specialised algorithm could be used (e.g. for lexicographic or bi-objective optimisation).

\bibliography{bib}
\bibliographystyle{splncs03}

\end{document}